\documentclass{article}

\usepackage{bm}
\usepackage{dsfont}
\usepackage{mathtools}
\usepackage{thmtools}
\usepackage{thm-restate}
\usepackage{tikz}
\usepackage{booktabs}       
\usepackage{amsfonts}       
\usepackage{nicefrac}       
\usepackage{bbding}
\usepackage{multicol}
\usepackage{multirow}
\usepackage{courier}
\usepackage{mathrsfs}
\usepackage{setspace}
\usepackage{tikz}

\usepackage{makecell}
\usepackage{bm}
\usepackage{appendix}
\usepackage{comment}
\usepackage{tablefootnote}
\usepackage{multirow}
\usepackage{hhline}%
\usepackage[font=footnotesize,labelfont=bf]{caption}
\usepackage{cases}
\usepackage{amssymb}

\usepackage{enumitem}

\newtheorem{lemma}{Lemma}
\newtheorem{proof}{Proof}

\usepackage{algorithm}
\usepackage{arydshln}

\usepackage[most]{tcolorbox}
\usepackage{fancyvrb}  
\usepackage{subcaption}

\newcommand{\am}{\mathop{\text{argmax}}}
\newcommand{\amin}{\mathop{\text{argmin}}}

\newcommand{\indicator}[1]{\mathds{1}[#1]}

 \usepackage[dblblindworkshop, final]{neurips_2026}
\workshoptitle{VLM4RWD}

\usepackage[utf8]{inputenc} 
\usepackage[T1]{fontenc}    
\usepackage{hyperref}       
\usepackage{url}            
\usepackage{booktabs}       
\usepackage{amsfonts}       
\usepackage{nicefrac}       
\usepackage{microtype}      
\usepackage{xcolor}         

\title{Efficient Omni-modal Large Language Model Inference via Per-Query Modality Selection}

\author{
  Shuqi Liu \quad Yuzhou Cao \quad Luke Ong \\
  College of Computing and Data Science, Nanyang Technological University \\
  \texttt{\{shuqi005,yuzhou002\}@e.ntu.edu.sg \quad
   luke.ong@ntu.edu.sg} \\
}

\begin{document}
\maketitle

\begin{abstract}
Omni-modal large language models (Omni-LLMs) can jointly process a broad range of modalities, such as text, audio, and video, enabling stronger multimodal understanding and reasoning capabilities, but often at substantial inference cost.
We observe that the benefit of additional modalities varies across queries, indicating that a fixed modality configuration can be inefficient.
Motivated by this, we study query-wise modality selection, which keeps the underlying Omni-LLM fixed while adaptively selecting which modalities to provide for each query.
We formulate the problem as an accuracy--cost optimization problem and develop a cost-sensitive learning approach for modality selection. 
We evaluate our approach on WorldSense using Qwen2.5-Omni and Qwen3-Omni under three inference-cost measures: input tokens, latency, and monetary cost. 
Across both models, query-wise modality selection consistently improves the accuracy--cost trade-off and can even outperform 
the standard setting in which all available modalities are provided to the Omni-LLM.
For Qwen2.5-Omni, our cost-sensitive approach matches or exceeds this full-modality accuracy while requiring only approximately 41\% of the input tokens, 39\% of the latency, and 61\% of the monetary cost.
\end{abstract}

\section{Introduction}
Multimodal Large Language Models (MLLMs) \cite{alayrac2022flamingo,liu2023visual,dai2023instructblip,qwen2.5-vl} have evolved from systems designed for specific modality combinations, such as text--image or text--video, to Omni-LLMs \cite{jiang-etal-2025-specific,hong2025worldsense,zhou2025daily,unobench,qwen25,Qwen3-Omni} that can jointly process text, images, audio, and video.
By jointly leveraging complementary information across these modalities, Omni-LLMs can improve multimodal understanding and reasoning, but processing a broader set of modalities also incurs greater inference overhead \cite{ding2026omnisift,DBLP:conf/cvpr/FuDLLRZWZSZCLLZ25,qwen25,Qwen3-Omni}.

To reduce this cost, a line of studies has explored improving Omni-LLM efficiency by pruning, compressing, or reallocating multimodal tokens \cite{yang2026omniselect,huang2026omnidelta,ding2026omnisift,tao2026omnizip,xin2026stage}. 
However, these methods still generally assume that all available modalities are provided to the model, and therefore cannot fully avoid the computational overhead of processing long audio and video inputs. 
This leaves another opportunity for reducing inference cost largely unexplored. 
Some queries may already be answered correctly using only a subset of the available modalities, such that the remaining audio or video inputs provide little or no additional gain. 
As long as these modalities are still processed by the Omni-LLM, however, their associated computational cost remains.
A more direct solution is therefore to identify unnecessary modalities and omit them from the model input entirely.
Unlike reducing the number of tokens after a modality is included, such omission can completely eliminate the Omni-LLM processing cost associated with that modality, without degrading model performance.
To investigate whether such a reduction is possible, we first ask a more fundamental question: \textbf{does every query require all available modalities to obtain a correct answer from an Omni-LLM?}

Our empirical results give a negative answer to the question: the modalities required to answer a query can vary substantially across queries.
Figure~\ref{fig: two examples} illustrates this behavior with two examples from WorldSense~\cite{hong2025worldsense}.
In Figure~\ref{audiosufficienr}, audio alone provides sufficient evidence to answer the question, while adding video changes Qwen2.5-Omni's \cite{qwen25} correct prediction to an incorrect one. 
This suggests that additional modalities can sometimes be unnecessary and may even be misleading during inference.
In contrast, the question in Figure~\ref{videorequired} asks for the color of a door, which cannot be resolved without the visual information.
Together, these examples show that the utility of video can vary considerably across queries: it may be unnecessary or even detrimental for some queries, while essential for others.
This query-wise variation is particularly important, since processing video incurs significant additional cost. On WorldSense, Qwen2.5-Omni processes about 3,000 input tokens per query on average with text and audio, whereas including video increases the average input size to about 40,000 tokens, which is more than a 13$\times$ increase.
These observations motivate \textit{per-query modality selection}: using additional modalities only when their benefit justifies the inference cost.

Motivated by these observations, we aim to adaptively determine which modalities should be used for each query to achieve a better trade-off between prediction accuracy and inference cost. As illustrated in Figure~\ref{fig: fixed}, using a fixed modality configuration leads to different operating points: Text+Audio (TA) incurs lower inference cost but achieves lower accuracy, whereas Text+Audio+Video (TAV) improves accuracy at a substantially higher cost. 
Since video is unnecessary for some queries and can even hurt prediction, selectively including it only when beneficial offers the potential to reduce inference cost with little or no loss in accuracy, and in some cases even improve prediction performance. 
We therefore study per-query modality selection, which dynamically selects the modality configuration for each query before invoking the Omni-LLM.

In this work, we study query-wise modality selection for efficient Omni-LLM inference. As shown in Figure~\ref{fig: selection}, our goal is to select the modality configuration for each query by balancing the potential performance gain from adding costly modalities, such as video, against the resulting inference cost.
We formulate this objective as a constrained accuracy--cost optimization problem and derive a cost-sensitive learning approach through its Lagrangian relaxation. 
Empirical results show that our modality-selection approach can match the accuracy of full-modality inference at substantially lower cost, demonstrating the potential of adaptive modality use for more efficient Omni-LLM inference.

\begin{figure}
 \begin{subfigure}[t]{0.5\linewidth}
        \centering
        \includegraphics[width=\linewidth]{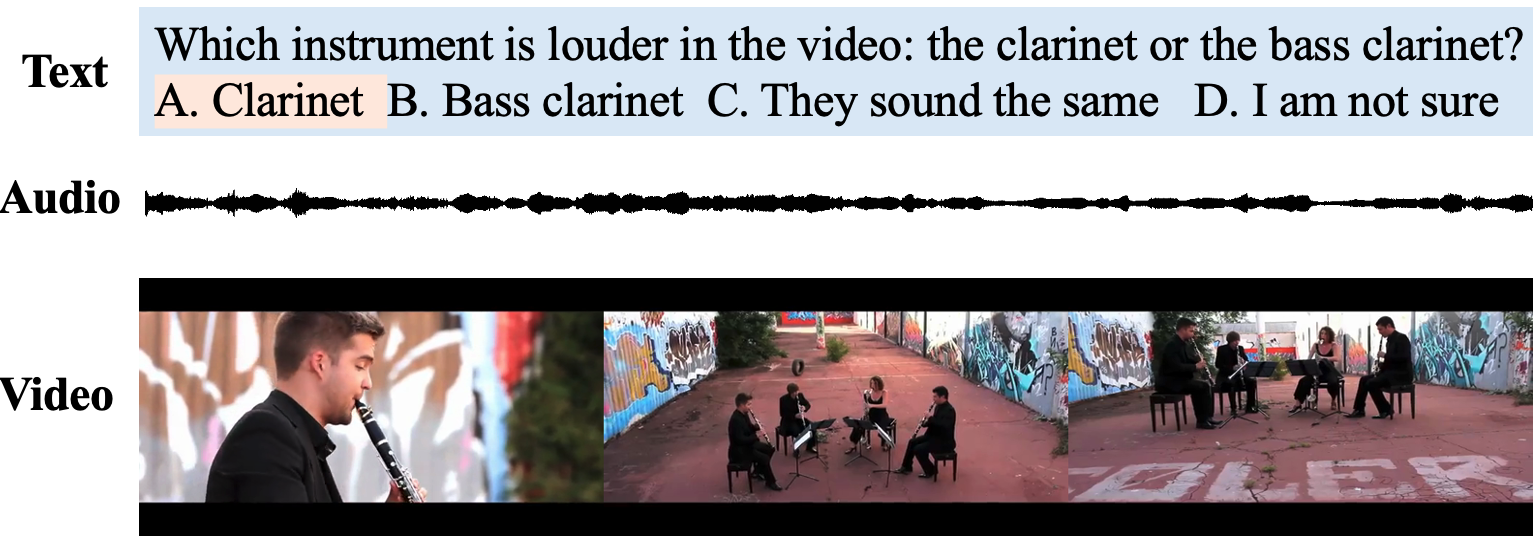}
        \caption{Audio-sufficient Query: TA$\checkmark$ TAV$\times$}
        \label{audiosufficienr}
\end{subfigure}
 \begin{subfigure}[t]{0.5\linewidth}
        \centering
        \includegraphics[width=\linewidth]{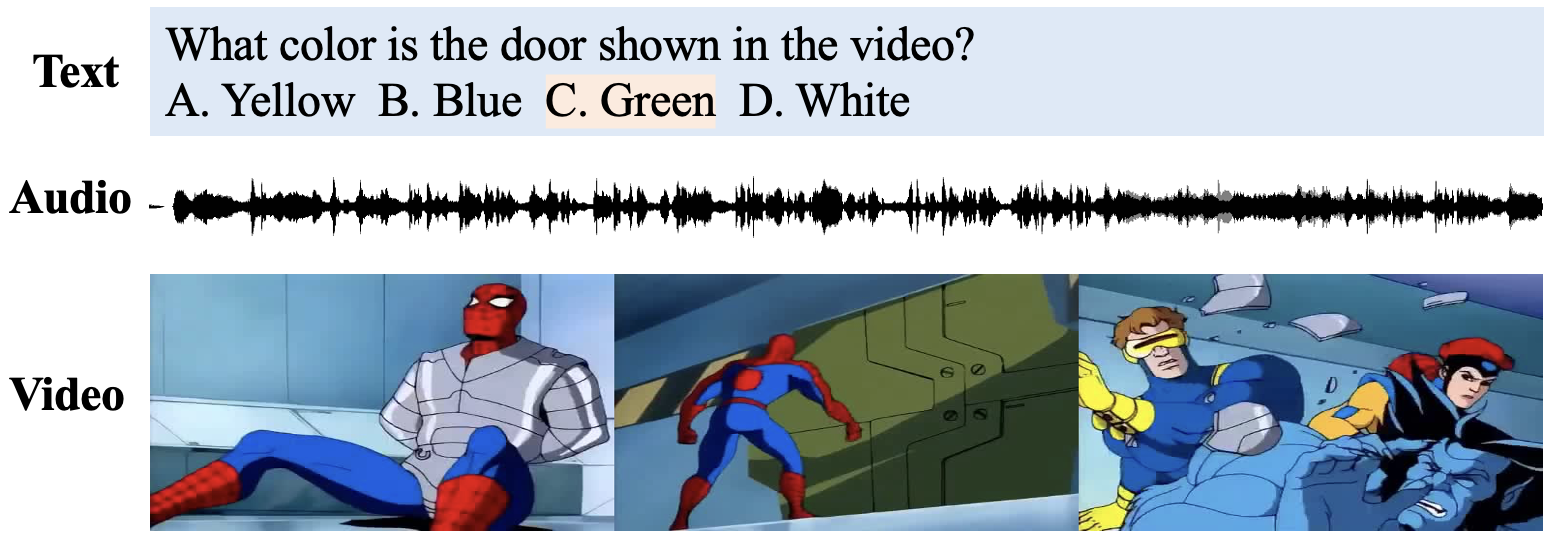}
        \caption{Video-required Query: TA$\times$ TAV$\checkmark$}
        \label{videorequired}
\end{subfigure}
    \caption{Examples from WorldSense illustrating query-dependent modality requirements. 
    Checkmarks $\checkmark$ and crosses $\times$ indicate prediction correctness using Qwen2.5-Omni.}
    \label{fig: two examples}
    \vspace{-7pt} 
\end{figure}

\begin{figure}
 \begin{subfigure}[t]{0.4\linewidth}
        \centering
       \includegraphics[width=\linewidth]{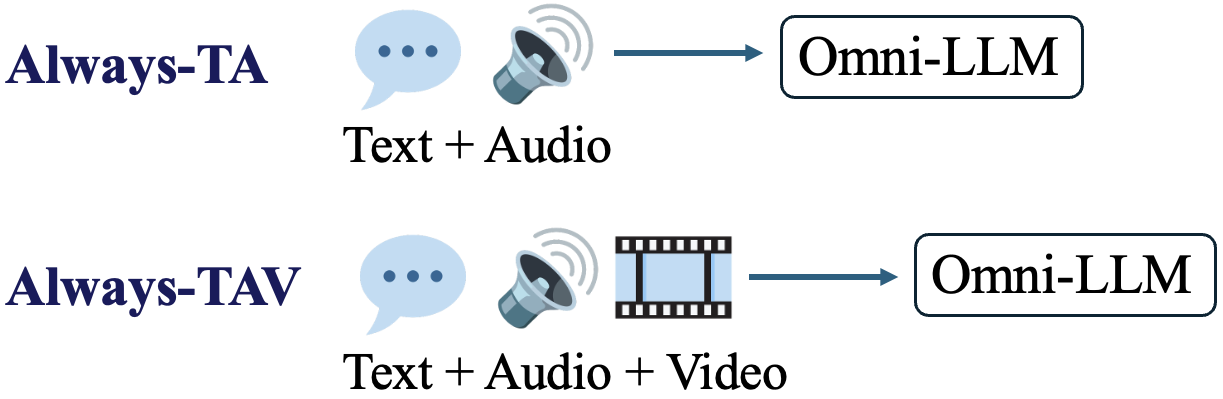}
        \caption{Fixed Modality Configurations}
        \label{fig: fixed}
\end{subfigure}
\hfill
 \begin{subfigure}[t]{0.57\linewidth}
        \centering
        \includegraphics[width=\linewidth]{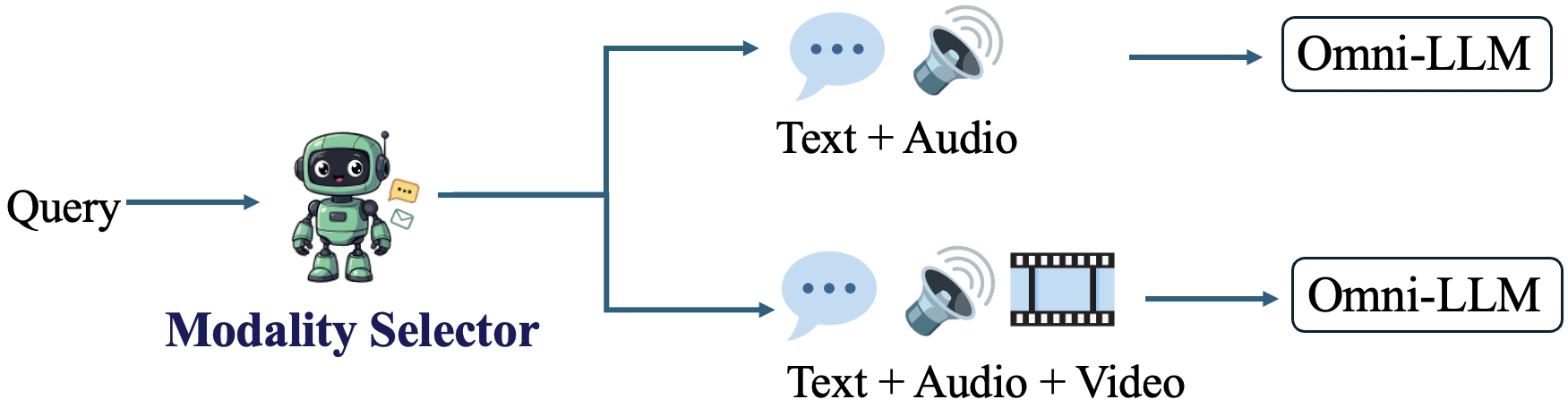}
        \caption{Per-query modality selection}
        \label{fig: selection}
\end{subfigure}
    \caption{Comparison between fixed modality configurations and per-query modality selection. (a) In fixed-modality inference, a single input configuration is used for every query, either Text+Audio (TA) or Text+Audio+Video (TAV). (b) In our framework, a modality selector determines for each query whether to use TA or TAV before invoking the Omni-LLM for prediction.}
    \label{fig: fixed vs selection}
    \vspace{-8pt} 
\end{figure}

To summarize, our contributions are:
\begin{itemize}
\vspace{-2pt}
    \item We formulate query-wise modality selection for Omni-LLM inference as an accuracy--cost optimization problem, where the modality configuration is adapted to each query.
    \item We develop a cost-sensitive modality selector derived from this formulation, and adapt representative LLM routing and cascading strategies to establish baselines for query-wise modality selection.
    \item We evaluate our methods on WorldSense with Qwen2.5-Omni and Qwen3-Omni, showing consistent accuracy--cost improvements over fixed-modality and adaptive baselines.
\end{itemize}

\section{Related Work}
\paragraph{Token Compression for Omni-LLMs.}
MLLMs process large numbers of visual tokens, motivating extensive work on token selection and pruning for efficient inference \cite{yang2025visionzip,liu-etal-2025-video,yao2025timechat,tao2025dycoke,shang2025llava,zhang2024sparsevlm}.
Omni-LLMs extend multimodal understanding to diverse inputs such as image, audio, and video \cite{qwen25,Qwen3-Omni,jiang-etal-2025-specific}. For audio--visual inputs, previous studies show that effective compression must account for modality-specific information distributions while preserving cross-modal correspondence~\cite{cao2026omnifocus,su2026omniscope}.
Recent work therefore develops audio--visual compression methods for Omni-LLMs, ranging from token pruning and adaptive budget allocation~\cite{yang2026omniselect,huang2026omnidelta,ding2026omnisift,tao2026omnizip,zhong2026omni,yoo2026out} to layer-wise and hybrid compression within the LLM~\cite{xin2026stage,park2026omnidrop,su2026omnipack,lee2026deferred}.

\paragraph{Modality Selection.}
Prior work studies modality selection across different modalities and settings.
\cite{JMLR:v25:23-0439} selects a fixed subset from available modalities under a cardinality constraint and uses the same subset across samples.
Input-dependent methods instead decide which modalities to process for each input; in video recognition, \cite{panda2021adamml,weng2023hcms} dynamically activate RGB/appearance, motion, and audio streams to reduce computation.
At the system-serving level, \cite{hu-etal-2024-mosel} selects modality configurations based on request-level accuracy and latency requirements together with system load.
Recently, \cite{luo2026chain} studies adaptive modality use in Omni-LLMs through agentic orchestration of modalities, input topology, and reasoning pathways, with the goal of improving multimodal reasoning robustness.

\paragraph{LLM Routing and Cascade.}
LLMs vary in capability and inference cost: stronger models often perform better but are also more expensive.
This trade-off motivates adaptive model selection, routing easier queries to cheaper models and invoking stronger ones only when needed \cite{chen2023frugalgpt,ong2025routellm}.
The idea of adapting system action to each input is closely related to learning to defer \cite{madras,DBLP:conf/icml/MozannarS20,gagan,liu2024mitigating,liu2026more,Okati2,DBLP:conf/icml/VermaN22,DBLP:conf/icml/CharusaieMSS22,DBLP:conf/aistats/VermaBN23,DBLP:conf/aistats/MozannarLWSDS23,mao2024realizable,wei2024exploiting,DBLP:conf/aistats/TailorPVMN24,charusaie2024unifying,palomba2024causal,DBLP:conf/icml/MontreuilCNO25, tian2026deferredseg,pesenti2026too,pugnana2026deferring,giacchetta2026should,zhu7292758coverage,soen2026density,montreuil2026beyond,mestoukirdi2026expected} and learning with rejection \cite{chow,bartlettrej,cortes2,charoenphakdee,DBLP:conf/nips/CaoC0GG00S22,DBLP:conf/nips/CaoM0W023,mao2023ranking,narasimhan2024learning,narasimhan2024plugin,cao2026establishing} in machine learning, where uncertain predictions are deferred to human experts or rejected, respectively.
LLM model selection is an input-dependent choice among candidate LLMs, typically via routing \cite{ding2024hybrid,routerbench2024,ding2025bestroute,song-etal-2025-irt,somerstep2025carrot,jitkrittum2026universal,tsiourvas2025causal,DBLP:journals/corr/abs-2607-08940,shirkavand2025costaware} 
or cascading \cite{aggarwal2024automix,zhang2024efficient,mahmood2026routing,dekoninck2025a,jitkrittum2023does,rabanser2026gatekeeper,narasimhan2025faster,gupta2024language}. Routing selects a single model for each query, whereas cascading sequentially invokes stronger models when needed.

\section{Query-wise Modality Selection}
As discussed in Section~1, extra modalities should be used only when their potential benefit justifies the additional inference cost. We therefore formulate query-wise modality selection for Omni-LLM inference as an accuracy--cost optimization problem and develop a cost-sensitive modality selector.

\subsection{Modality Selection Problem Formulation}
\label{sec:setup}
\paragraph{Data Generating Distribution} 
We consider the input-level modality selection problem for a fixed pretrained Omni-LLM, denote by $f$. 
Let $\bm{q}=(\bm{z}_1,\dots,\bm{z}_M)$ be an omni-modal input query, where $\bm{z}_m \in \mathcal{Z}_m$ is the input associated with modality $m\in \mathcal{M}:=[M]$, such as, text, image, audio, or video. 
Denote by $y\in \mathcal{Y}:=[K]$ the ground-truth answer for query $\bm{q}$.
We assume that the query-answer pair $(\bm{q},y)$ is drawn from an unknown joint distribution $\mathcal{D}$ over $\mathcal{Z}_1 \times \cdots \times \mathcal{Z}_M\times \mathcal{Y}$, and the density of this distribution is $p(\bm q, y)$.
We assume the training data set consists of $n$
i.i.d. samples drawn from $\mathcal{D}$.

\paragraph{Technical Notations} For a modality subset $S\subseteq\mathcal{M}$, let $\widetilde{\bm {q}}_{S}$ denote the version of query $\bm{q}$ restricted to the modalities indexed by $S$,
\begin{align}
    \widetilde{\bm q}_{S}
    :=
    \bm q\big|_{\mathcal S}
    =
    (\bm{z}_m)_{m\in S}.
\end{align}
Specifically, if
$S=\{1,3\}$, then $\widetilde{\bm q}_{S} = \bm (\bm{z}_1,\cdots,\bm{z}_M)\big|_{\mathcal \{1,3\}}= (\bm{z}_1,\bm{z}_3).$
For $(\bm q, y)$, we further define
\begin{align}
\text{Acc}_{S}^f(\bm{q}) :=
\indicator{f(\widetilde{\bm{q}}_S)=y}, \quad
\text{Cost}_{S}^f(\bm{q}) := \sum_{j\in S} c_f(\bm{z}_j)
\end{align}
to be the prediction correctness of Omni-LLM $f$ on the modality-restricted query $\widetilde{\bm{q}}_{S}$ and its corresponding inference cost, respectively. 
Here, $c_f(\bm z_j)$ denotes the inference cost incurred by $f$ when processing the modality input $\bm z_j$.
The inference cost is summed over the selected modalities
and may be measured in terms of input tokens, monetary cost, or inference
latency.
When $S=\mathcal M$, the setting reduces to standard full-modality inference, where all available modalities are provided to $f$.

\paragraph{Problem Setup} 
We construct a query-specific decision representation $\bm{h}(\bm{q})\in\mathcal{H}$ using the information available at the time of modality selection.
This design is flexible, allowing $\bm{h}(\bm{q})$ to incorporate different sources of information, such as text, metadata, or lightweight modality features.
A modality selector $r:\mathcal{H}\rightarrow 2^{\mathcal{M}}$ then maps $\bm{h}(\bm{q})$ to a modality subset $r(\bm{h}(\bm{q})) \subseteq \mathcal{M}$.

Our goal is to learn a modality selector $r$ that maximizes the expected prediction accuracy under an average cost budget $B>0$:
\begin{align}
\label{eq:constrained}
\max_{r}\quad
\mathbb{E}_{p(\bm{q},y)}[\text{Acc}_{r(\bm h (\bm q))}^{f}(\bm{q})]
\quad
\text{s.t.}\quad
\mathbb{E}_{p(\bm{q})}
[\text{Cost}_{r(\bm h (\bm q))}^{f}(\bm{q})] \leq B.
\end{align}

To characterize the optimal query-wise modality decision, we consider the Lagrangian relaxation of Eq.~\eqref{eq:constrained}. For a trade-off parameter $\lambda\geq 0$, this yields the following target risk:
\begin{align}
R_{\lambda}(r)
= \mathbb{E}_{p(\bm q,y)} 
\left[\indicator{f(\widetilde{\bm q}_{r(\bm h (\bm q))})\neq y}
+
\lambda \sum_{j\in r(\bm h (\bm q))} c_f(z_j)
\right]
\label{eq:target_risk}
\end{align}

\begin{lemma}[Bayes-Optimal Solution]
    \label{lem:bayes_optimal}
    Given the representation function $\bm{h}$, 
    the selector minimizing the target risk~\eqref{eq:target_risk} is given by
    \begin{equation}
        \label{eq:bayes_router}
       r^*(\bm{h}(\bm{q})) = 
        \am\limits_{S\subseteq \mathcal{M}}
          \bigl\{
          \text{Acc}_{S}^f(\bm{q}) - \lambda\, \text{Cost}_{S}^f(\bm{q})\bigr\},
    \end{equation}
    where ties are broken according to a fixed arbitrary rule.
    The corresponding optimal modality-restricted query is $\widetilde{\bm{q}}_{r^*(\bm{h}(\bm{q}))}=
    (\bm{z}_m)_{m\in r^*(\bm{h}(\bm{q}))}$.
\end{lemma}
\begin{proof}
    We decompose the target risk by conditioning on $q$:
    \begin{align*}
        R_{\lambda}(r) &= \mathbb{E}_{\bm{q}}\left[\mathbb{E}_{ y \mid \bm{q}}\left[\indicator{f(\widetilde{q}_{r(\bm{h}(\bm{q}))}) \neq y}
        + \lambda\ \sum_{j\in r(\bm{h}(\bm{q}))}c_f(z_j) \right]\right] \\
       & = \mathbb{E}_{\bm q}\left[1-\text{Acc}_{r(\bm{h}(\bm{q}))}^f (\bm{q}) + \lambda\, \text{Cost}_{r(\bm{h}(\bm{q}))}^f (\bm{q})\right].
    \end{align*}
   Since $\bm h$ is fixed, the routing decision $r(\bm h(\bm q))$ is determined solely by $\bm q$. Thus, the integrand is minimized pointwise for each $\bm q$ by choosing $$r^* (\bm{h}(\bm{q})) = \amin_{S\subseteq \mathcal{M}} \{1-\text{Acc}_{S}^f (\bm{q}) + \lambda\, \text{Cost}_S^f  (\bm{q})\} = \am_{S\subseteq \mathcal{M}} \{\text{Acc}_S^f (\bm{q}) - \lambda\, \text{Cost}_S^f (\bm{q})\}.$$
\end{proof}
Here, $\lambda$ controls the trade-off between prediction performance and inference cost. 
The optimal rule~\eqref{eq:bayes_router} selects, for each query, the modality subset that achieves the best balance between these two quantities. 
Thus, the optimal subset is determined by its query-specific prediction performance and inference cost, rather than by an a priori preference for larger or smaller modality subsets.

\subsection{Cost-Sensitive Router}
For a new query, the prediction outcome and cost associated with each candidate modality subset are unknown before invoking $f$.
Thus, the optimal rule in Eq.~\eqref{eq:bayes_router} is not directly available at test time. 
We therefore learn a router $r$ that predicts the modality subset from the query representation $\bm h(\bm q)$, using the training set
$\{(
\bm q_i,
\{
\text{Acc}_{S}^f (\bm q_i),
\text{Cost}_{S}^f (\bm q_i)
\}_{S\subseteq\mathcal M}
)
\}_{i=1}^{n}$.

This is naturally a cost-sensitive classification problem, since different routing errors can have different consequences in terms of prediction accuracy and inference cost.
For a fixed $\lambda$, define the $\lambda$-utility of selecting $S$ as
$
    u^{\lambda}_{S}(\bm q)=\text{Acc}_S^f(\bm q)-\lambda\text{Cost}_S^f (\bm q).
$
Accordingly, the target label for a training query $\bm q_i$ is the modality subset
$
    S^*_{\lambda}(\bm{q}_i)=\am_{S\subseteq \mathcal{M}} u^{\lambda}_{S}(\bm{q}_i),
$
which is exactly the subset preferred by the optimal rule in Eq.~\eqref{eq:bayes_router}.
We further assign each query a weight $w_i$
, defined as the utility gap between the optimal subset and its best alternative:
\begin{align}
    w_i=u^{\lambda}_{S^*_\lambda (\bm q_i)}(\bm q_i)-\max\limits_{S\neq S^*_\lambda (\bm q_i)} u^{\lambda}_S(\bm q_i).
\end{align}
A larger $w_i$ indicates that a routing error has a greater impact on the accuracy--cost objective and should therefore receive greater emphasis during training.

Following the standard reduction of cost-sensitive classification to weighted classification \cite{zadrozny2003cost}, we learn the router using a weighted cross-entropy surrogate. Let $\bm g(\bm h(\bm q)) = (g_S(\bm h (\bm q)))_{S\subseteq \mathcal{M}}$ denote the router's score vector, indexed by the candidate modality subsets. We optimize
\begin{align}
      \widehat{\bm{g}}=\amin_{\bm g}\ \frac{1}{n} \sum_{i=1}^{n}\frac{w_i}{\bar w}\,
\ell_{\text{CE}}\bigl(\bm g(\bm h(\bm q_i)),\,S^*_i\bigr),
\quad  
\label{eq: cs_loss}
\end{align}
where $\bar  w=\tfrac1N\textstyle\sum_i w_i$ and 
$\ell_{\text{CE}}$ is the standard cross-entropy loss.
Thus, the router is trained to predict the optimal modality subset while placing greater emphasis on queries for which a routing error has a larger impact on the original accuracy--cost objective.

At test time, the prediction outcome and cost of the candidate subsets are no longer required. The learned router only takes $\bm h(\bm q)$ as input and selects the subset with the largest predicted score:
\begin{equation}
\widehat r\bigl(\bm h(\bm q)\bigr)
=
\am_{S\subseteq\mathcal M}
\widehat{g}_S\bigl(\bm h(\bm q)\bigr).
\label{eq:recover_router}
\end{equation} 
Since the target labels and sample weights both depend on $\lambda$, varying $\lambda$ produces routers corresponding to different accuracy--cost trade-offs and traces the corresponding operating curve.

\section{Experiments}
In this section, we evaluate our query-wise modality-selection framework on WorldSense to assess its empirical effectiveness. Additional experimental details and results are provided in the Appendix.

\subsection{Models and Dataset}
For the underlying Omni-LLM, we conduct experiments with Qwen2.5-Omni-7B~\cite{qwen25} and Qwen3-Omni-30B-A3B~\cite{Qwen3-Omni} treating each as the fixed underlying Omni-LLM $f$ and training a modality selector.
We evaluate on WorldSense~\cite{hong2025worldsense}, an omni-modal question-answering benchmark, where $M=3$ and each query $\bm q=(\bm z_1,\bm z_2,\bm z_3)$ consists of text, audio, and video.

Rather than considering all possible modality subsets, we focus on a simplified binary selection problem: given text and audio, should video be additionally provided for a particular query? 
This setting allows us to examine whether query-wise modality selection can already yield a meaningful accuracy--cost trade-off. Accordingly, we consider two modality subsets:
\begin{align}
    S_{\text{TA}} = \{1,2\}, 
    S_{\text{TAV}} = \{1,2,3\}.
\end{align}
Thus, the always-TA and always-TAV selectors are defined as
$r_{\text{TA}}(\cdot)\equiv S_{\text{TA}}$ and
$r_{\text{TAV}}(\cdot)\equiv S_{\text{TAV}}$.
We use a single linear layer as the modality selector, with the BGE-M3~\cite{chen2024bge} embedding of the textual input $z_1$ as the decision representation $\bm h$. Thus, the selector determines whether to include video using only textual information available before Omni-LLM inference.

\subsection{Evaluation Metrics}
\paragraph{Accuracy and Cost.} The system performance is evaluated using average accuracy and average cost over the test set:
\begin{align}
    \widehat{\text{Acc}}(r)=
    \frac{1}{n_{\text{test}}}
    \sum_{i=1}^{n_{\text{test}}} \text{Acc}_{r(\bm h (\bm q_i))}^f (\bm{q}_i),\quad
    \widehat{\text{Cost}}(r)=
    \frac{1}{n_{\text{test}}}
    \sum_{i=1}^{n_{\text{test}}} \text{Cost}_{r(\bm h (\bm q_i))}^f (\bm{q}_i).
\end{align}
We consider three cost measures: input tokens,
inference latency, and estimated monetary cost. 
The first two are directly
measured from model inference, while monetary cost is defined using modality-dependent token prices, with details provided in Appendix.

\paragraph{Accuracy--Cost AUC.}
For methods with tunable thresholds or hyperparameters, we sweep them to obtain accuracy--cost operating points and evaluate the area under the resulting Pareto curve over the common cost range from always-TA to always-TAV.
For each selector $r$, we construct $\widetilde{A}r(C)$ by linearly interpolating its Pareto-optimal points, anchored at the always-TA point. If the curve does not reach the always-TAV cost, we extend its final accuracy horizontally to $c{\text{TAV}}$. We then define the normalized AUC as
\begin{align}
\label{eq: AUC}
    \text{AUC}(r)
    =
    \frac{   \displaystyle
        \int_{c_{\text{TA}}}^{c_{\text{TAV}}}
        \left(
            \widetilde{A}_r(C)-a_{\text{TA}}
        \right)\,dC}{
        \left(c_{\text{TAV}}-c_{\text{TA}}\right)
        \left(a_{\text{TAV}}-a_{\text{TA}}\right)} ,
\end{align}
where $c_j= \widehat{\text{Cost}}(r_j), a_j = \widehat{\text{Acc}}(r_j), j\in\{\text{TA},\text{TAV}\}$.
A larger AUC indicates a better overall trade-off, 
assigning higher scores to methods that achieve larger accuracy gains at lower costs.

\paragraph{Cost Reduction at Always-TAV Accuracy.}
To quantify the efficiency gain from modality selection, we measure how much cost can be reduced while maintaining the same accuracy as always-TAV. 
For each method and each cost measure, we find the point on its accuracy--cost Pareto curve whose accuracy matches that of always-TAV and denote the corresponding cost by $c_r^{\text{match}}$.
The relative cost reduction is then computed as
\begin{align}
\label{eq: cost saving}
\Delta\% (r)= \frac{c_{\text{TAV}}-c_r^{\text{match}}}{c_\text{TAV}} \times 100.
\end{align}
A larger positive $\Delta\%$ indicates greater cost savings while maintaining the accuracy of full-modality inference.

\subsection{Routing and Cascade Baselines}
We consider two paradigms for query-wise modality selection: routing and cascading. 
Routing selects a modality configuration directly for each query, 
$
    r(\bm h(\bm q)) \in \{S_{\text{TA}},S_{\text{TAV}}\},
$
before invoking the Omni-LLM.
In contrast, a cascade first processes the query with TA and then decides whether to stop or escalate to TAV based on the intermediate model output. 
We adapt representative methods from both paradigms to our modality-selection setting.

\paragraph{Routing Baselines.}
All routing baselines make the TA-versus-TAV decision before Omni-LLM inference, but differ in how this decision is obtained.
\begin{itemize}
    \item \textbf{Random Router} selects TAV with probability $p$ and TA with probability $1-p$,
\begin{align}
r( \bm h (\bm q))=
\begin{cases}
S_{\text{TAV}}, & \text{with probability } p,\\
S_{\text{TA}}, & \text{with probability } 1-p.
\end{cases}
\end{align}
Sweeping $p\in[0,1]$ yields its accuracy--cost curve.
    
    \item \textbf{LLM-as-a-Judge Router.} This router uses an external LLM to predict whether video is needed for a given query. 
    For this baseline, we use the textual query $z_1$ as the decision representation, i.e., $\bm h(\bm q)=z_1$. 
    The judge outputs a binary label $J(\bm h(\bm q))\in\{0,1\}$, where $J(\bm h(\bm q))=1$ indicates that additional video information is needed. The routing rule is
    \begin{align}
    r(\bm h (\bm q))=
    \begin{cases}
    S_{\text{TAV}}, & \text{if } J(\bm h(q))=1,\\
    S_{\text{TA}}, & \text{otherwise}.
    \end{cases}
    \end{align}
    We evaluate Gemini 2.5 Pro \cite{comanici2025gemini} and Gemini 2.5 Flash-Lite \cite{comanici2025gemini} as the external judge. The judge prompt is provided in Appendix. 
    \item \textbf{Estimate Then Predict.} This learning-based baseline first estimates the quality and cost of each modality configuration from the query representation, and then selects the configuration with the highest estimated utility. Specifically, for each configuration $j\in\{\text{TA},\text{TAV}\}$, we learn predictors $\hat a_j(\bm h(\bm q))$ and $\hat c_j(\bm h(\bm q))$ from the observed accuracy and cost in the training set. At test time, the router evaluates the estimated utility of each configuration and selects
    \begin{align}
    \hat r(\bm h (\bm q))=
    \am_{j\in\{\text{TA},\text{TAV}\}}
    \left[
    \hat a_j(\bm{h}(\bm q))
    -\lambda \hat c_j (\bm{h}(\bm q))
    \right],
    \end{align}
    where $\lambda\geq 0$ controls the trade-off between quality and cost. Varying $\lambda$ produces the corresponding accuracy--cost trade-off curve.
\end{itemize}

\paragraph{Cascade Baselines.}
Cascade first processes every query with the cheaper TA configuration and then decides whether to escalate to TAV.
Let $\text{Conf}(\bm q)$ denote a confidence score obtained after the TA-stage inference. 
Given a threshold $\tau\in [0,1]$, the cascade stops at TA when the current prediction is sufficiently confident and otherwise escalates to TAV:
\begin{align}
   r(\bm h (\bm q))=
    \begin{cases}
    S_{\text{TAV}}, & \text{Conf}(\bm q)<\tau,\\
    S_{\text{TA}}, & \text{Conf}(\bm q)\ge \tau.
    \end{cases} 
\end{align}
For an escalated query, the Omni-LLM $f$ is first invoked with $S_{\mathrm{TA}}$ and then again with $S_{\mathrm{TAV}}$, resulting in the cumulative inference cost of both runs.
We consider two confidence measures. \textbf{Verbalized confidence} prompts the Omni-LLM to report a confidence score with its TA prediction and is applicable to both black-box and white-box models. 
\textbf{Logit-based confidence} uses the model's output logits over the answer options and is therefore limited to white-box Omni-LLMs. Sweeping $\tau$ yields the corresponding accuracy--cost trade-off curve.

\begin{figure}[tbp]
    \begin{subfigure}[t]{0.33\linewidth}
        \centering
        \includegraphics[width=\linewidth]{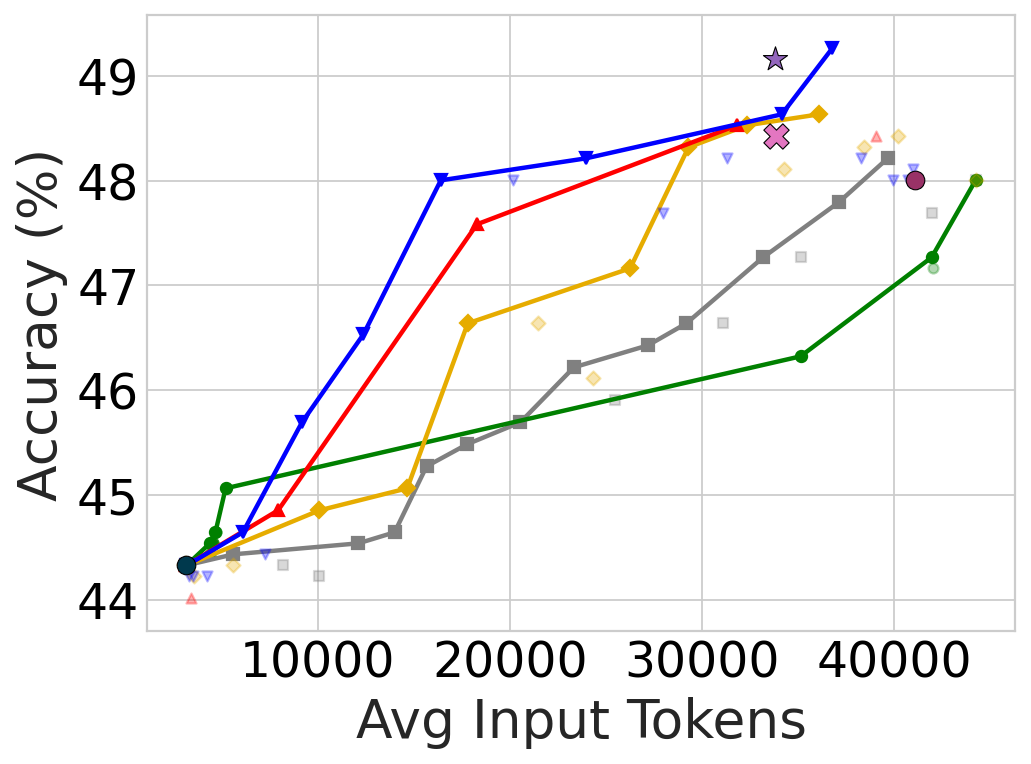}
        \captionsetup{font=footnotesize}
        \caption{Qwen2.5-Omni: Input Tokens}
    \end{subfigure}
    \begin{subfigure}[t]{0.33\linewidth}
        \centering
        \includegraphics[width=\linewidth]{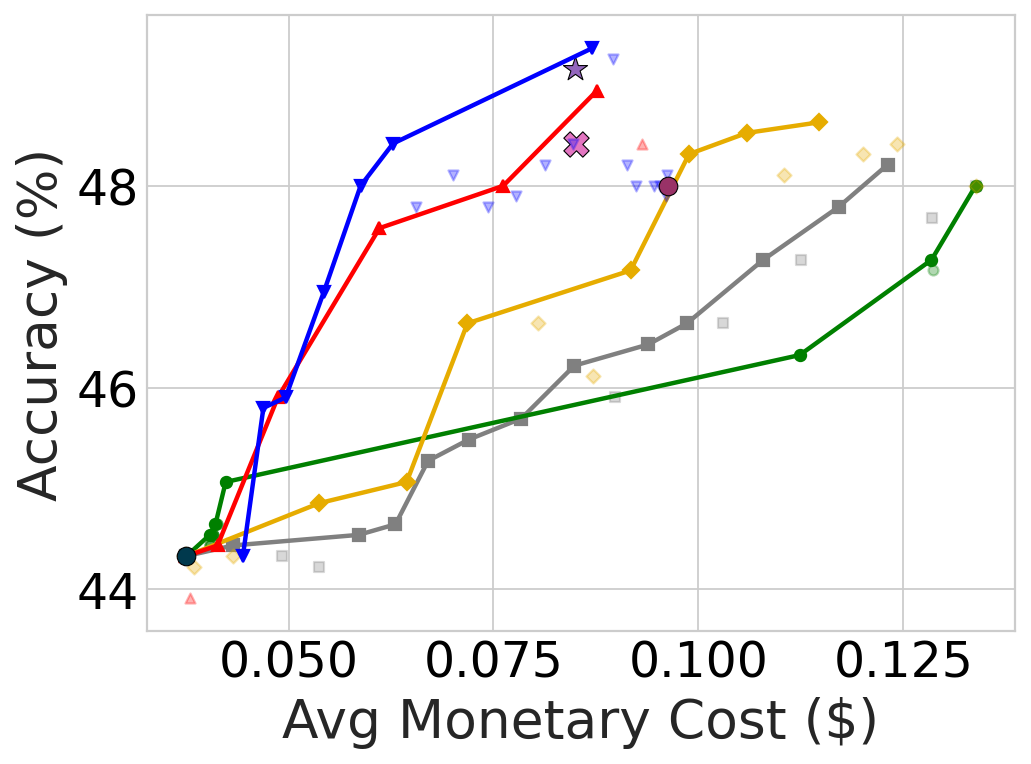}
        \captionsetup{font=footnotesize}
        \caption{Qwen2.5-Omni: Monetary}
    \end{subfigure}
    \begin{subfigure}[t]{0.33\linewidth}
        \centering
        \includegraphics[width=\linewidth]{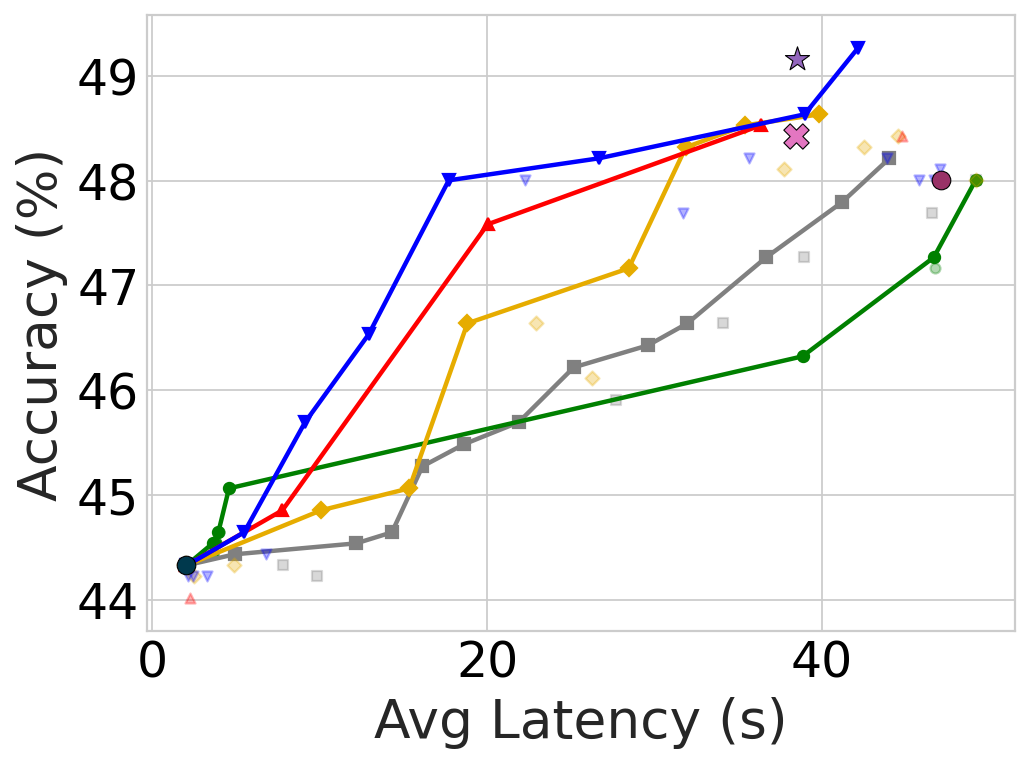}
        \captionsetup{font=footnotesize}
        \caption{Qwen2.5-Omni: Latency}
    \end{subfigure}
    \\
        \begin{subfigure}[t]{0.33\linewidth}
        \centering
        \includegraphics[width=\linewidth]{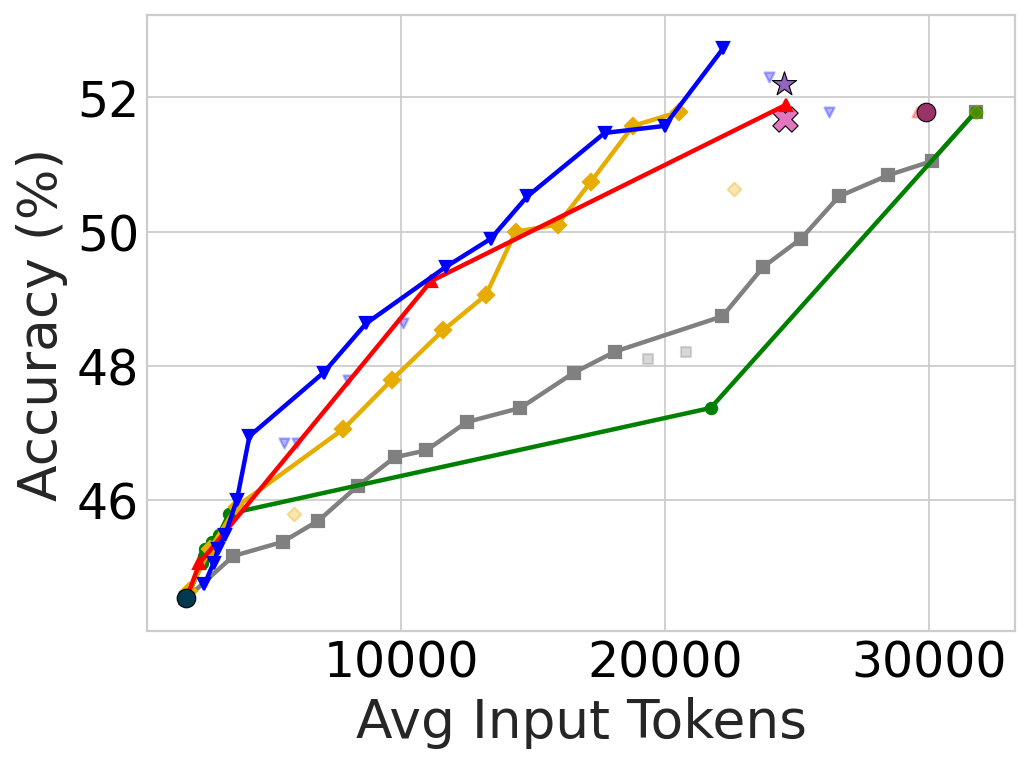}
        \captionsetup{font=footnotesize}
        \caption{Qwen3-Omni: Input Tokens}
    \end{subfigure}
    \begin{subfigure}[t]{0.33\linewidth}
        \centering
        \includegraphics[width=\linewidth]{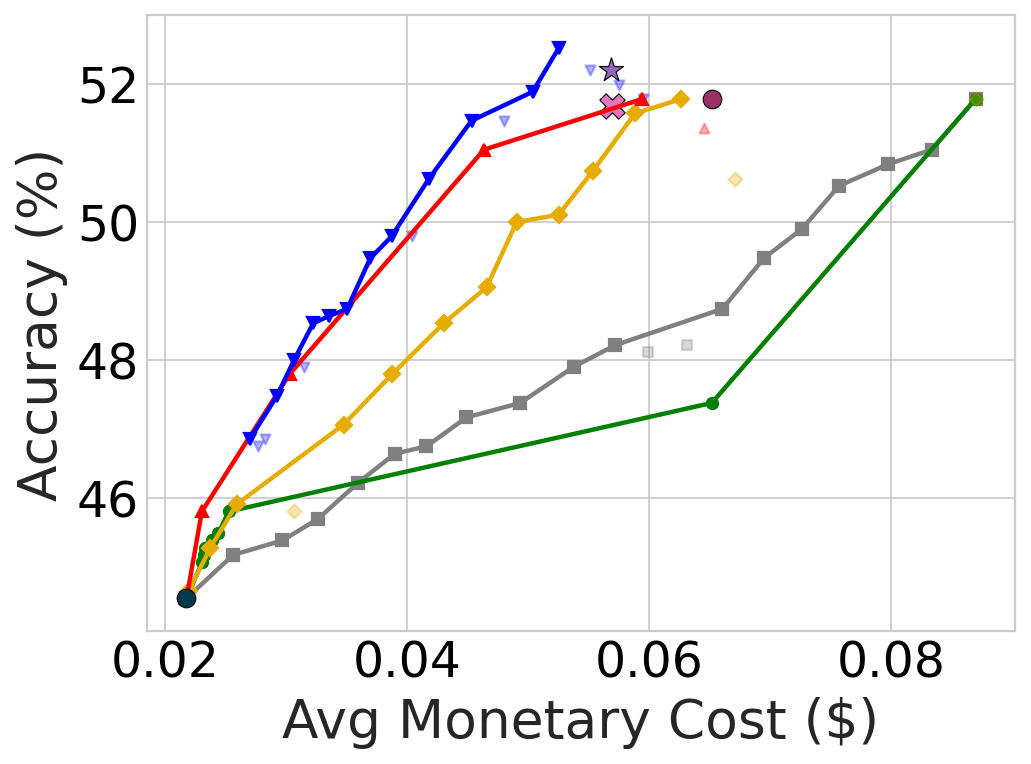}
        \captionsetup{font=footnotesize}
        \caption{Qwen3-Omni: Monetary}
    \end{subfigure}
    \begin{subfigure}[t]{0.33\linewidth}
        \centering
        \includegraphics[width=\linewidth]{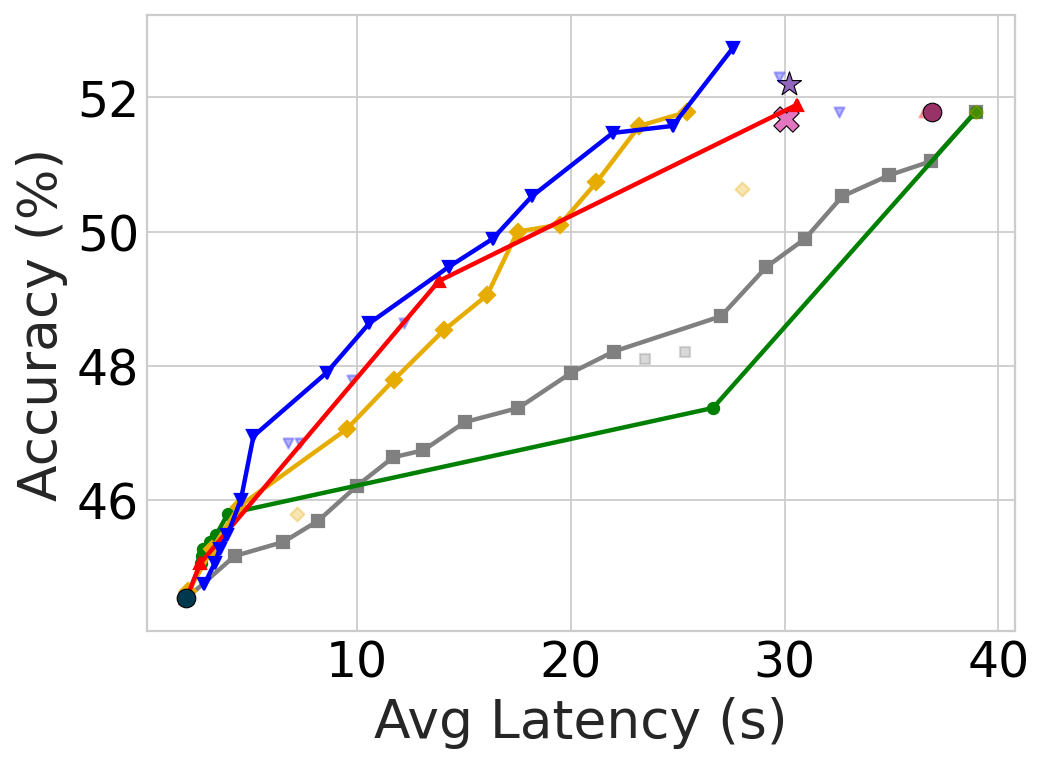}
        \captionsetup{font=footnotesize}
        \caption{Qwen3-Omni: Latency}
    \end{subfigure}
    \\
    \begin{subfigure}[t]{\linewidth}
    \centering
        \includegraphics[width=0.9\linewidth]{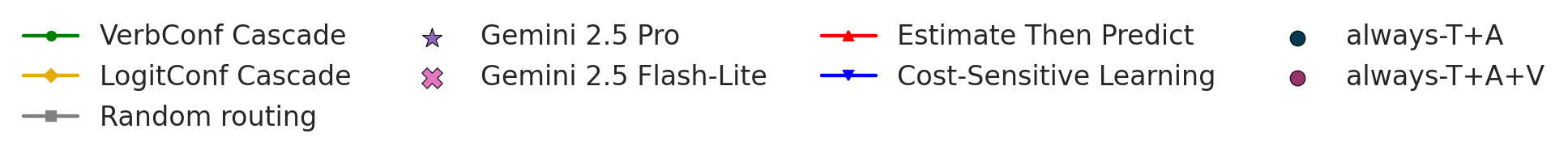}
    \end{subfigure}
\caption{
Accuracy--cost trade-off for query-wise modality selection on WorldSense using Qwen2.5-Omni (top) and Qwen3-Omni (bottom).
The three columns measure cost by average input tokens, monetary cost, and latency per query, respectively. 
Semi-transparent markers denote operating points obtained by sweeping each method's decision threshold or hyperparameter, while solid lines show the corresponding Pareto frontiers. 
Filled circles denote the Always-TA and Always-TAV endpoints, and random routing serves as a non-adaptive reference. 
}
\label{Acc-Cost Pareto Frontiers}
\end{figure}

\subsection{Experimental Results}
\paragraph{Accuracy--Cost Tradeoff}
In Figure~\ref{Acc-Cost Pareto Frontiers}, we report the accuracy--cost trade-offs of different modality-selection methods on WorldSense using Qwen2.5-Omni and Qwen3-Omni under three cost measures: input tokens, monetary cost, and latency. 
Cost-Sensitive Learning achieves the best overall trade-off across both Omni-LLMs and all three cost measures, attaining higher accuracy at lower cost over most of the operating range. Estimate-Then-Predict also consistently improves over the cascade baselines. The LLM-as-a-Judge routers provide competitive single operating points, but are generally less efficient than the learning-based routers. Among the cascades, logit-based confidence yields moderate gains, whereas verbalized confidence performs close to or below random routing over much of the cost range.

In Table~\ref{tab: AUC}, we further report the AUC score defined in Eq.~\eqref{eq: AUC}. 
Consistent with the accuracy--cost curves in Figure~\ref{Acc-Cost Pareto Frontiers}, cost-sensitive learning achieves the highest AUC under all three cost measures for both Qwen2.5-Omni and Qwen3-Omni.
Specifically, its AUCs are 0.9265/0.9125/0.9878 on Qwen2.5-Omni and 0.8670/0.8587/0.9422 on Qwen3-Omni for input tokens, latency, and monetary cost, respectively.
The second-best baseline varies across two models: Estimate-Then-Predict ranks second under all three costs on Qwen2.5-Omni, while on Qwen3-Omni, LogitConf Cascade performs better for token and latency, and Estimate-Then-Predict remains stronger for monetary cost. 
Overall, cost-sensitive learning maintains a clear advantage across all six model-cost settings.

\begin{table}[tbp]
\centering
\caption{Normalized accuracy--cost AUC of different modality-selection methods across three cost measures on WorldSense.}
\label{tab: AUC}
\begin{tabular}{lcccccccc}
\toprule
Model & \multicolumn{3}{c}{Qwen2.5-Omni} & \multicolumn{3}{c}{Qwen3-Omni}\\
\midrule
Method &  TokenIn & Latency & Money&  TokenIn & Latency & Money \\
\midrule
{Random Router}&0.5046& 0.4987& 0.4920&0.4986&0.4985& 0.4985 \\ 
VerbConf Cascade &0.4355 & 0.4327& 0.4965&0.4110&0.4193& 0.4195 \\
LogitConf Cascade &0.6963 & 0.6967& 0.6043&0.7143&0.7108& 0.7133\\
Estimate Then Predict & 0.8178 &0.8047& 0.8992&0.6941&0.6860& 0.7895\\
Cost-Sensitive Learning&\textbf{0.9265} &\textbf{0.9125}& \textbf{0.9878}&\textbf{0.8670}&\textbf{0.8587}&\textbf{0.9422}\\
\bottomrule
\end{tabular}
\end{table}

\begin{table}[tbp]
\centering
\caption{Cost reduction at Always-TAV accuracy.
For each method, we report the input-token, monetary, and latency costs at which its accuracy--cost Pareto frontiers  reach the Always-TAV accuracy (48.00\% for Qwen2.5-Omni; 51.79\% for Qwen3-Omni). Costs are averaged over test queries, and $\Delta$\%~\eqref{eq: cost saving} denotes the relative cost reduction compared with Always-TAV.}
\label{tab: cost saving}
\begin{tabular}{llcccccc}
\toprule
Model & Method & TokenIn & $\Delta$\%  & Latency & $\Delta$\% & Money& $\Delta$\% \\ 
\midrule
\multirow{5}{*}{Qwen2.5}&Always-TAV &  41050 &0.00 & 47.13 &0.00 & 0.0963 &0.00\\
&VerbConf Cascade  & 44224 &-7.73& 49.17 & -4.33 & 0.1338 & -38.94  \\
&LogitConf Cascade & 29289 &28.65 & 31.81 & 32.51 & 0.0999 & -3.74\\
&Estimate Then Predict& 28099 & 31.55  &    31.74 & 32.65 & 0.0762 & 20.87 \\
&Cost-Sensitive Learning & \textbf{16877} & \textbf{58.89} & \textbf{18.35} & \textbf{61.07} & \textbf{0.0588} & \textbf{38.94} \\
\midrule
\multirow{5}{*}{Qwen3}& Always-TAV & 29874&0.00 & 36.89 &0.00 & 0.0652 & 0.00\\
&VerbConf Cascade & 31744 &-6.26& 38.93 &-5.53 & 0.0870 & -33.44\\
&LogitConf Cascade  & 20527&31.29 & 25.46 &30.98 & 0.0626& 3.99 \\
&Estimate Then Predict &26159&12.44 &32.40& 12.17 &0.0594& 8.90 \\
&Cost-Sensitive Learning & \textbf{19502} &\textbf{34.72} & \textbf{24.14} & \textbf{34.56} & \textbf{0.0490}& \textbf{24.85}\\
\bottomrule
\end{tabular}
\end{table}

\paragraph{Cost Reduction at Always-TAV accuracy.}
We further compare the methods at the accuracy of Always-TAV, measuring how much inference cost can be reduced while maintaining full-modality accuracy. 
As shown in Table \ref{tab: cost saving}, Cost-Sensitive Router achieves the lowest cost under all three measures for both Qwen2.5-Omni and Qwen3-Omni. 
On Qwen2.5-Omni, it reduces input tokens, latency, and monetary cost by 58.89\%, 61.07\%, and 38.94\%, respectively; on Qwen3-Omni, the corresponding reductions are 34.72\%, 34.56\%, and 24.85\%. 
Estimate-Then-Predict also achieves savings across all three measures on both models, although the gains are smaller on Qwen3-Omni. 
The logit-based cascade reduces token and latency costs on both models, but its monetary cost slightly increases on Qwen2.5-Omni. 
In contrast, the verbalized-confidence cascade is more expensive than Always-TAV across all three measures for both models, suggesting that its confidence estimates do not provide reliable escalation decisions and that the cumulative cost of cascading further increases inference cost.

\section{Conclusion}
We studied query-wise modality selection for Omni-LLMs, where the modality configuration is selected for each query to avoid unnecessary use of costly modalities. We formulated this problem as a constrained accuracy--cost optimization problem and developed a cost-sensitive learning approach for modality selection. Experiments on WorldSense with Qwen2.5-Omni and Qwen3-Omni show that our method consistently achieves favorable accuracy--cost trade-offs across input-token, latency, and monetary costs, while matching full TAV accuracy at substantially lower inference cost.
These results show that query-wise modality selection is an effective approach to reducing Omni-LLM inference cost while preserving accuracy.

\section*{Acknowledgements}
The project is partially supported by National Research Foundation, Singapore, under the NRF RSS Scheme NRF-RSS2022-009 and the Singapore Global AI VP grant AIVP-2024-002. Yuzhou Cao is supported by Google PhD Fellowship program.

\bibliographystyle{abbrv}

\appendix

\section{Details of Experiments}
\subsection{Models and Dataset}
Both backbones are frozen and queried in direct-prediction format: greedy decoding (\texttt{do\_sample=False}) with a 16-token thinker budget and at most 32 new tokens, so a response is a single option letter. Qwen2.5-Omni-7B runs in \texttt{bfloat16} and Qwen3-Omni-30B-A3B in \texttt{auto} precision.
Each model is evaluated on a single NVIDIA H200 GPU.
Video frame rate and frame budget are left at each model's own processor default rather than forced to a common value, since aligning raw frame counts across two visual tokenisers would put at least one model outside its training regime; audio is resampled to 16\,kHz for both. 
The prompt names the attached evidence in a single header line (\emph{``\dots based on the given audio''} vs.\ \emph{``\dots based on the given video (with audio)''}).

The query text is formed by concatenating the question with all four answer options and is encoded once by a frozen \texttt{BAAI/bge-m3} encoder into a 1024-dimensional, $\ell_2$-normalized dense vector.
The learning-based router $r$ is a single linear layer trained with AdamW (learning rate $1e-4$, weight decay $3e-3$) for 100 epochs with batch size 64.
The accuracy--cost curve is obtained by sweeping the cost weight over 100 points in $[0,1]$. 
We use a $70:30$ train--test split, and all methods are evaluated on the same test set for a consistent comparison.

\subsection{Evaluation Metric}
\paragraph{Monetary cost.} 
The number of input tokens serves as a pricing-agnostic proxy for inference cost, but treats tokens from different modalities equally.
In practice, commercial omni-modal APIs may assign different prices to text, audio, and video tokens. We therefore additionally evaluate each method using an estimated monetary cost that accounts for modality-specific pricing. Let
\begin{align}
   \bm{\pi}=(\pi_{\text{T}},\pi_{\text{A}},\pi_{\text{V}})
\end{align}
denote the per-token prices for text, audio, and video, respectively. For a query $\bm q$, the monetary cost under selector $r$ is
\begin{align}
    \widetilde{C}(\bm{q};r) 
    =
    \pi_{\text{T}}n(\bm{q},\text{T})
    +
    \pi_{\text{A}}n(\bm{q},\text{A})
    +
    \mathbb{I}\!\left\{r(\bm{q})=S_{\text{TAV}}\right\}
    \pi_{\text{V}}n(\bm{q},\text{V}),
\end{align}
where $n(\bm{q},m)$ denotes the number of input tokens from modality $m$.
In our experiments, we use the 
Alibaba Cloud Model Studio international-service pricing for
Qwen3.5-Omni-Plus,
$\bm{\pi}=(1.40, 11.00, 1.40)\times 10^{-6}$ USD per input token for text, audio, and video, respectively, and report the average monetary cost over the test set.

\subsection{Routing and Cascade Baselines}

\paragraph{LLM-as-a-Judge Router.} We prompt an external LLM to predict, based only on the question and answer options, whether visual information is required. Its binary prediction is directly used to route a fixed target Omni-LLM between TA and TAV inference.

We study whether an off-the-shelf LLM can serve as a question-only router for an omni-modal cascade. Given only the multiple-choice question and its answer options, the judge predicts whether the sample can be answered by the cheap TA or requires visual information from the full TAV. The judge does not see the video, audio, metadata, or ground-truth answer; thus, it estimates the \textbf{modality requirement} of a question rather than judging answer correctness. Its binary decision is then used as a hard routing signal: label 0 sends the sample to the text+audio stage (TA), while label 1 defers it to the text+audio+video stage (TAV).

To examine how routing quality depends on judge capability and cost, we compare two LLM judges with different profiles: Gemini 2.5 Pro and Gemini 2.5 Flash Lite. Both judges use the same instruction format, decoding setting, and structured binary output.

\begin{tcolorbox}[
    colback=gray!3,
    colframe=black!70,
    boxrule=1.2pt,
    arc=2mm,
    left=6mm,
    right=6mm,
    top=4mm,
    bottom=4mm,
    width=\linewidth
]
\begin{Verbatim}[fontsize=\small,baselinestretch=1.05]
You are an annotator. Given a multiple-choice question about a video,
decide whether the question can be answered from text + audio alone 
(speech, sound, music) or whether it REQUIRES the visual content of 
the video (e.g., objects on screen, colors, spatial relations, 
gestures, scene appearance, on-screen text, visual events).

You see only the question and the answer options. You do NOT see the
video, the audio, any metadata, or the correct answer.

Output exactly one binary label:
0  = text + audio SUFFICIENT (no visual frames needed)
1  = video REQUIRED (visual frames necessary to answer)

Return ONLY a JSON object with fields: label (0 or 1) and reason 
(<= 20 words).
\end{Verbatim}
\end{tcolorbox}

\paragraph{Verbalized Cascade} Elicitation prompt:
\begin{tcolorbox}[
    colback=gray!3,
    colframe=black!70,
    boxrule=1.2pt,
    arc=2mm,
    left=6mm,
    right=6mm,
    top=4mm,
    bottom=4mm,
    width=\linewidth
]
\begin{Verbatim}[fontsize=\small,baselinestretch=1.05]
  Please answer the following question: <...>
  Please choose only one answer from the following options:
    A. ...
    B. ...
    C. ...
    D. ...
    
  First output the letter (A, B, C, or D) only, then on the next 
  line output 'Confidence: N' where N is an integer 0-100 expressing 
  how sure you are.
\end{Verbatim}
\end{tcolorbox}

\paragraph{Logit-based Cascade}

We further use the model's first-token logits as a confidence signal for cascade 
routing. 
For each multiple-choice question, we prompt the model to output one answer letter from
$\{A,B,C,D\}$. During generation, we take the logit vector from the first decoding step, since the first generated token is expected to be the answer letter.

Let $l_A,l_B,l_C,l_D$ be the logits assigned to the four answer tokens $A,B,C,D$ at this
first decoding step. We convert these four logits into probabilities by applying a softmax
only over the answer options:
\[
p(i \mid \bm{x})
=
\frac{\exp(l_i)}
{\sum_{j \in \{A,B,C,D\}} \exp(l_j)},
\quad i \in \{A,B,C,D\}.
\]
The logit-based confidence is then defined as the largest probability among the four options:
\[
\text{conf}_{\text{logit}}(\bm{q})
=
\max_{a \in \{A,B,C,D\}} p(a \mid \bm{q}).
\]
In cascade, if this confidence is higher than a threshold $\tau$, we accept the prediction
from the cheaper TA stage. Otherwise, we defer the sample to the more expensive
TAV stage. Therefore, the threshold $\tau$ controls the trade-off between cost and accuracy.

\begin{table}[t]
\centering
\caption{Summary of the modality-selection methods evaluated in our experiments.}
\label{tab:method_summary}
\begin{tabular}{lcccc}
\toprule
Method & Paradigm & $\bm h(\bm q)$ & Training & Additional Requirement \\
\midrule
Oracle
&Routing & Omni-LLM correctness & No & Ground-truth\\
Random Router
& Routing & Random score & No & None \\

Task-Type Router
& Routing & Task metadata & No & Metadata \\

Estimate then Predict
& Routing & Text embedding & Yes &Text Encoder\\

Cost-Sensitive Router
& Routing & Text embedding & Yes &Text Encoder \\

LLM-as-a-Judge
& Routing & External LLM judgment & No & External LLM \\

Verbalized Cascade
& Cascade & Self-reported confidence & No & None \\

Logit-Based Cascade
& Cascade & Prediction logits & No & Logit access \\
\bottomrule
\end{tabular}
\end{table}

\section{Additional Experimental Results}
\subsection{Potential of Query-wise Modality Selection}
\paragraph{Oracle Selector}To quantify the potential of query-wise modality selection, we consider an oracle that chooses between TA and TAV for each query using the ground-truth answer. 
For each query--answer pair $(\bm q,y)$, the oracle
evaluates all candidate modality subsets and selects
\begin{align}
\label{eq: oracle}
    S^{\text{oracle}}_{\lambda}(\bm q,y)
    \in
    \am\limits_{S\subseteq \mathcal{M}}
    [\text{Acc}_{S}^f (\bm{q})-
        \lambda \text{Cost}_{S}^f (\bm{q})].
\end{align}
On WorldSense with Qwen2.5-Omni-7B, the oracle achieves a maximum $\widehat{\text{Acc}}$ of $59.03\%$ at a $\widehat{\text{Cost}}$ $8,970$ input tokens per query. Compared with always-TAV, the oracle improves accuracy by $11.03$ percentage points while reducing input-token cost by $78.1\%$, indicating substantial headroom for query-wise modality selection. The full oracle accuracy--cost frontier is provided in the Figure~\ref{fig:oracle_frontier}.

\paragraph{A Simple Task-Type Router.}
WorldSense groups its 26 task types into three task domains:
\begin{itemize}
\item \textbf{Recognition:} Action Counting, Anomaly Recognition, Attribute Recognition, Audio Counting, Audio Recognition, Audio Source Localization, Event Recognition, Object Counting, Object Existence Recognition, Scene Recognition, and Temporal Localization.
\item \textbf{Reasoning:} Attribute Reasoning, Audio Change, Causal Reasoning, Emotion Change, Hallucination, Object State Change, Relation Reasoning, and Temporal Prediction.
\item \textbf{Understanding:} Event Sorting, Human Emotions, Human Interaction, Human-object Interaction, Spatial Relation, Text and Diagram Understanding, and Video Emotions.
\end{itemize}

To illustrate both the potential of query-wise modality selection and the flexibility of the decision representation, we construct a router using only task-type metadata. 
WorldSense assigns each query to one of 26 fine-grained task types under 
three domains: Recognition, Understanding, and Reasoning. 
We rank the task types by the empirical accuracy gain of TAV over
TA. For each cutoff $k$, the router assigns TAV to queries whose
$\bm{h}^{\text{task}}(\bm{q})$ belongs to the top-$k$ task types and T+A otherwise.
\begin{align}
    r(\bm{h}^{\text{task}}(\bm{q}))=
    \begin{cases}
     S_{\text{TAV}}, \text{ if } \bm{h}^{\text{task}} (\bm{q}) \text{ is among the top-$k$ task types},\\
     S_{\text{TA}}, \text{ otherwise}.
\end{cases}
\end{align}
Sweeping $k$ from $0$ to $26$ yields the accuracy--cost curve in Figure~\ref{fig:task_type_curve}.
Despite its simplicity, the task-type
router achieves a meaningful accuracy--cost improvement on the test set.

\begin{figure}
\centering
    \begin{subfigure}[t]{0.35\linewidth}
        \centering
        \includegraphics[width=\linewidth]{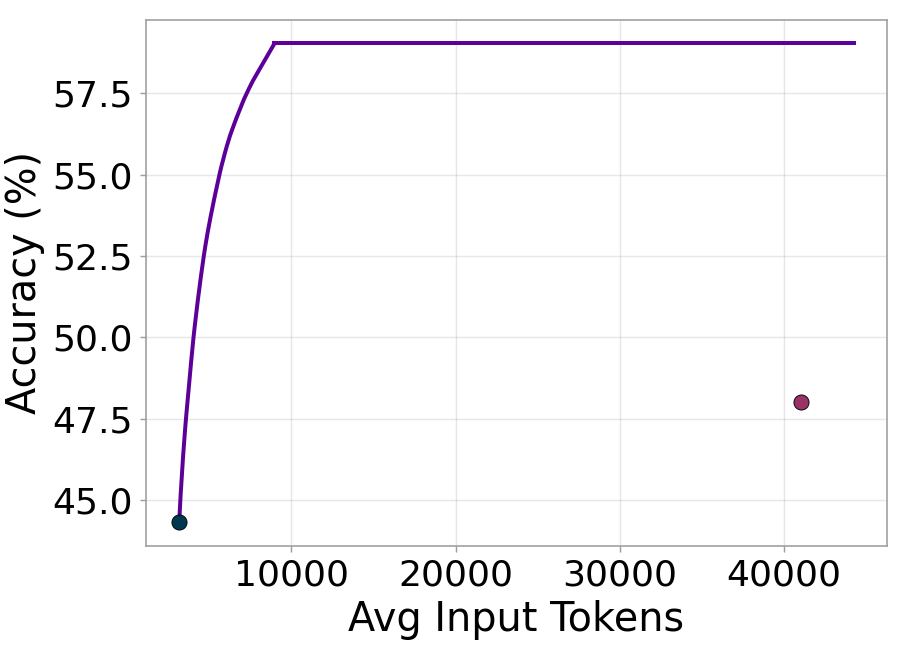}
        \caption{Oracle Router}
        \label{fig:oracle_frontier}
    \end{subfigure}
    \begin{subfigure}[t]{0.35\linewidth}
        \centering
        \includegraphics[width=\linewidth]{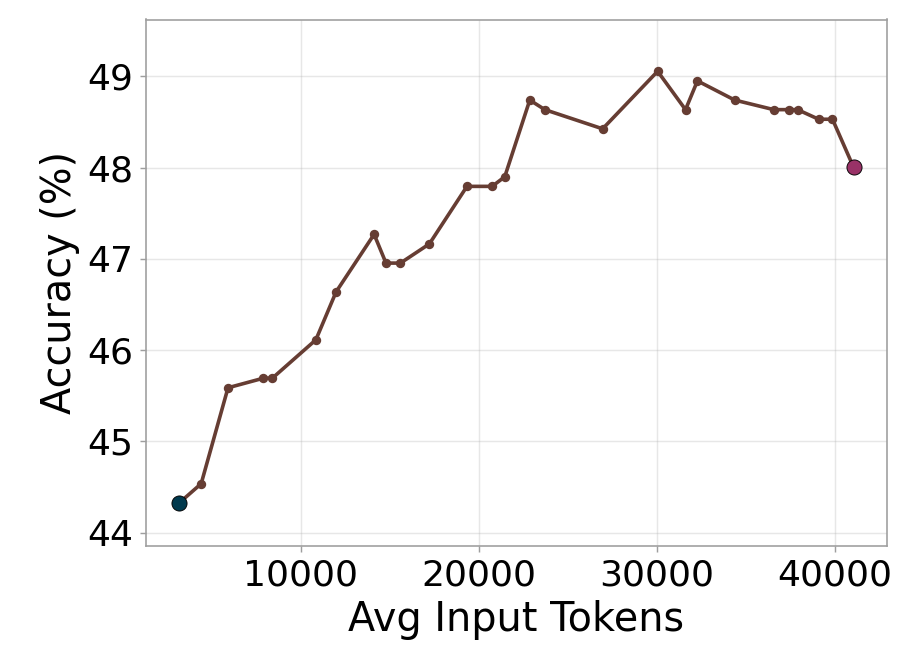}
        \caption{Task-Type Router}
        \label{fig:task_type_curve}
    \end{subfigure}
    \begin{subfigure}[t]{0.18\linewidth}
        \centering
        \raisebox{0.9cm}{
        \includegraphics[width=\linewidth]{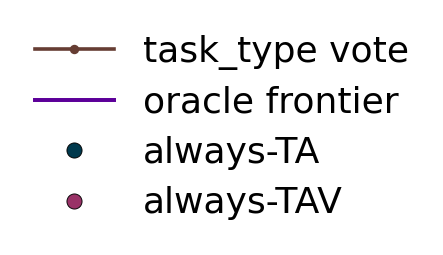}
        }
    \end{subfigure}
    \caption{Potential of query-wise modality selection on WorldSense using Qwen2.5-Omni.
(a) Accuracy--cost frontier of an oracle selector that chooses between T+A and TAV using ground-truth query outcomes.
(b) A simple task-type router that selects the modality configuration based only on the query's task type. The Always-TA and Always-TAV points are shown as reference configurations, and cost is measured by average input tokens per query.}
\end{figure}

\section{Limitations and Future Directions}
Our current study considers a simplified binary selection between
$S_{\mathrm{TA}}$ and $S_{\mathrm{TAV}}$, and focuses on a limited set of query
representations $\bm h(\bm q)$, including task-type metadata and text embeddings.
Natural future directions include extending modality selection to a larger set of candidate subsets and exploring richer multimodal representations, such as fused text, audio, and video embeddings or lightweight cross-modal signals.

Our modality-selection framework is also complementary to existing model-selection approaches, including LLM routing and cascading. A promising direction is therefore to jointly optimize both decisions, which model to invoke and which modalities to provide, for more efficient omni-modal question answering.

\end{document}